\documentclass[sigconf]{acmart}

\AtBeginDocument{%
  \providecommand\BibTeX{{%
    \normalfont B\kern-0.5em{\scshape i\kern-0.25em b}\kern-0.8em\TeX}}}

\copyrightyear{2023}
\acmYear{2023}
\setcopyright{acmlicensed}\acmConference[WWW '23]{Proceedings of the ACM
Web Conference 2023}{May 1--5, 2023}{Austin, TX, USA}
\acmBooktitle{Proceedings of the ACM Web Conference 2023 (WWW '23), May
1--5, 2023, Austin, TX, USA}
\acmPrice{15.00}
\acmDOI{10.1145/3543507.3583528}
\acmISBN{978-1-4503-9416-1/23/04}

\usepackage{amstext}
\usepackage{amsthm}

\usepackage{amssymb}
\usepackage{bm}
\usepackage{wrapfig}
\usepackage[utf8]{inputenc}
\usepackage{subfigure}
\usepackage{enumitem}
\usepackage{url}
\usepackage[ruled,linesnumbered]{algorithm2e}
\usepackage{color}

\newtheorem*{theorem*}{Theorem}
\newtheorem{theorem}{Theorem}[section]

\newtheorem{example}[theorem]{Example}
\newtheorem{proposition}[theorem]{Proposition}
\newtheorem{definition}[theorem]{Definition}

\begin{document}

\title{Near-Optimal Experimental Design Under the Budget Constraint in Online Platforms}


\author{Yongkang Guo}
\affiliation{%
  \institution{Peking University}
  \country{}}
\email{yongkang_guo@pku.edu.cn}

\author{Yuan Yuan}
\affiliation{%
  \institution{Purdue University}
  \country{}}
\email{yuanyuan@purdue.edu}

\author{Jinshan	Zhang}
\affiliation{%
  \institution{Zhejiang University}
  \country{}}
\email{zhangjinshan@zju.edu.cn}

\author{Yuqing Kong}
\affiliation{%
  \institution{Peking University}
  \country{}}
\email{yuqing.kong@pku.edu.cn}

\author{Zhihua Zhu}
\affiliation{%
  \institution{Tencent Technology (Shenzhen) Co., Ltd.}
  \country{}}
\email{zhihuazhu@tencent.com}

\author{Zheng Cai}
\affiliation{%
  \institution{Tencent Technology (Shenzhen) Co., Ltd.}
  \country{}}
\email{zhengcai@tencent.com}

\begin{abstract}
A/B testing, or controlled experiments, is the gold standard approach to causally compare the performance of algorithms on online platforms. 
However, conventional Bernoulli randomization in A/B testing faces many challenges such as spillover and carryover effects. 
Our study focuses on another challenge, especially for A/B testing on two-sided platforms -- budget constraints.
%
Buyers on two-sided platforms often have limited budgets, where the conventional A/B testing may be infeasible to be applied, partly because two variants of allocation algorithms may conflict and lead some buyers to exceed their budgets if they are implemented simultaneously. 
%
We develop a model to describe two-sided platforms where buyers have limited budgets. 
We then provide an optimal experimental design that guarantees small bias and minimum variance. 
Bias is lower when there is more budget and a higher supply-demand rate. We test our experimental design on both synthetic data and real-world data, which verifies the theoretical results and shows our advantage compared to Bernoulli randomization.
\end{abstract}



\maketitle

\section{Introduction}
\label{sec:intro}


A/B testing, also known as controlled experiments has been used as the gold standard approach to compare different algorithms on online platforms, due to its wide range of application scenarios and simple implementation \cite{kohavi2009controlled,kohavi2011unexpected,kohavi2014seven,ha2020counterfactual,bhagat2018buy}. In brief, A/B testing randomly divides a group of units into two subgroups, which often named as treatment group and control group. Then the platform allocates different variants of algorithms to units according to their subgroup and then receives the metrics of interest as an estimator for the effect of corresponding variants. Proper use of A/B testing reduces both potential financial and time losses for the platform caused by poor algorithms.
Despite the power of A/B testing, it may not perfectly obtain the precise estimation of the true effect of interest. The existence of spillover effects \cite{kwark2021spillover}, carryover effects \cite{bojinov2020design}, and the mismatch between short-term and long-term treatment effects \cite{dmitriev2016pitfalls} remain the major challenges to the credibility of A/B testing results. 

Our study focuses on online two-sided platforms with budget constraints. For example, on advertisement platforms where the two sides are advertisers and platform users respectively, each advertiser may have limited budgets and can only distribute their advertisements to certain users. The budget constraint is prevalent on other two-sided platforms, such as ride-hailing platforms where drivers can only take a limited number of passengers, or online shopping platforms where merchants have limited inventory. 


The platform can create advanced algorithms to address budget constraints, but these constraints pose challenges for standard A/B testing. Consider a scenario with 4 items and 2 buyers, where different experiment variations determine how items are assigned to buyers. Each buyer can only receive 2 items due to budget limitations, for example, an advertiser can only purchase a limited number of user views. Variant A assigns items 1 and 2 to buyer 1, and items 3 and 4 to buyer 2; variant B assigns items 1 and 3 to buyer 1, and items 2 and 4 to buyer 2. Running a standard A/B test on items results in a chance of 1/4 for each buyer to receive more than the 2 item budget limit. Similarly, running the test on buyers can lead to the simultaneous assignment of items 2 and 3 to both buyers. Budget constraints can make standard experimental design impractical.

A highly relevant challenge to A/B testing is ``interference'', or ``spillover effect'' \cite{blake2014marketplace,johari2022experimental,puelz2019graph,yuan2021causal,yuantwo}. 
It means that a unit outcome should not be affected by other units' treatment assignments according to Rubin's potential outcome model \cite{rubin2005causal}. 
Strictly speaking, interference is not the major concern in the motivating advertisement platform example: once the treatment assignments are carefully designed, one buyer's outcome should not be affected by other buyers' treatment assignments but only depends on the features of users who view his advertisement. 
However, our budget constraint problem shares similarities with the interference problem in terms of the need for improved experimental design and estimator construction. The connection between our work and interference will be discussed in the Related Work section.


In this paper, we address the budget constraint for A/B testing on two-sided platforms. 
We develop a model to describe any two-sided platforms with one-to-many relationships and one side has limited budgets. We then provide a near-optimal experimental design that guarantees small bias and minimum variance to compare two potential allocations. 
We solve the experimental design problem by convex optimization and extend it to the online setting.
We then demonstrate our approach with both synthetic and real-world data. As Bernoulli randomization is infeasible in our setting, we provide a modified version as a benchmark. 
Our results demonstrate that our algorithm outperforms the modified Bernoulli randomization, achieving a $20\%$ reduction in mean squared error (MSE). Furthermore, the performance of our algorithm is further improved under certain conditions such as a high ratio of budget to total costs and high similarity between two potential allocations.

Below is a summary of the contributions of this paper:
\begin{itemize}
    \item We formulate the budget constraint problem for A/B testing on two-sided platforms.
    \item We propose an experimental design to address this budget constraint problem, with bias upper bounded by $O(\left(m/n\right)^{2/3})$ and minimum variance.
    \item We propose a polynomial online algorithm, which takes into account the real-world scenario that items are not immediately known but come sequentially. Its computation complexity is $O(m)$ times the offline case. 
    \item We combine real-world data from Tencent and synthetic potential outcomes to demonstrate the validity of our algorithms. 
\end{itemize}

\section{Related Work}
\paragraph{Interference model}
Our study is highly relevant but different from the interference literature, such as~\cite{chen2015online,muchnik2013social,ugander2013graph,pouget2019testing,holtz2020reducing,aronow2017estimating}. We mention a few of the most relevant studies here. 
\citet{basse2016randomization} investigate the impact of interference on auction experiments with limited budgets, assuming that the experiment directly affects the potential bids of the buyers and then affects potential payments due to auction mechanisms. In contrast, our study does not consider background auction mechanisms and assume the payments are fixed. \citet{liu2021trustworthy} provide a budget-split mechanism to correct the cannibalization bias. They prove unbiasedness under limited interference and stable system assumptions. Our study instead focuses on a more detailed setting for each buyer and the experiment design in an online setting. \citet{johari2022experimental} analyze the bias in a two-sided market and introduces a two-sided randomization design based on a continuous time Markov model. They consider long-term stable states where budgets may recover, whereas our study focuses on the short-term states where the budget is fixed.

\paragraph{Budget management}
Budget management has been widely studied as a traditional optimization problem such as online matching problem \cite{mahdian2011online,feldman2010online,devanur2009adwords} or Adwords problem with concave return (APCR) \cite{chen2015dynamic,devanur2012online,huang2020adwords}. They often assume the price of each advertisement is pre-determined, which is similar to our model. But they aim to directly maximize the metric of interest while we want to minimize the variance of our matching experiment design. 
Some other studies focus on the game theory perspective~\cite{balseiro2015repeated, paes2016field,arnosti2016adverse}. They calculate the equilibrium in repeated auctions under budget constraints or design a mechanism with desirable properties. They mainly assume that advertisers strategically respond to different auction mechanisms. We focus on comparing two variants under budget constraints rather than designing the optimal mechanism and leave the strategies of advertisers as future direction.

\paragraph{Online optimization}
There has been great interest in online problems of random order adversaries. Solving packing linear programming is the most relevant work to this paper~\cite{agrawal2014dynamic,molinaro2014geometry,agrawal2014fast,kesselheim2018primal,gupta2016experts}. In packing linear programmings (LPs), there are many linear constraints representing the volumes of items placed in a package will not exceed its capacity, which is similar to the budget. \citet{agrawal2014dynamic} propose an online algorithm based on the linear objective and \citet{zhang2022online} extend it to more general objectives. We migrate the basic idea of solving proportional sub-problems for every item to our online algorithm and replace the objective function.

\section{Model}
In this section, we introduce our model for two-sided platforms under budget constraints and design experiments to examine the difference between two potential allocations. Here we take advertisement platforms as an example, though our setting can be generalized to other scenarios such as e-commerce and recommendation systems.

Suppose there are $m$ items and $n$ buyers. We focus on the one-to-many platforms so that the items will be allocated to the buyers and one item can be assigned to at most one buyer. Thus items and buyers form a bipartite graph $G$ where the edge $e_{ij}$ between item $i$ and buyer $j$ indicates item $i$ is allocated to buyer $j$. The platform has several ways to influence the allocation and different allocations will bring different utilities. For example, on advertisement platforms, buyers are the advertisers and items are user views.

We set budget constraints on the buyers' side. Each buyer $j$ needs to report their budget $b_j$ before the allocation, which represents the total amount the buyer can spend on one or multiple items. This reflects the scenario where advertisers do not want to exceed their budget when purposing exposures to their advertisements on a platform.

We assume that item $i$ will bring utility $u_{ij}$ and cost $c_{ij}$ for buyer $j$. They can be considered as the utility and cost of edge $e_{ij}$. On advertisement platforms, the utility summarizes the metrics of interest for both platforms and buyers. For the platforms, the utility is simply the revenue, which is the cost of buyers in advertisement platforms. For the buyers, the utility includes targets such as clicks or conversions, which can be generalized as the revenue brought by user views. The total utility may be a linear combination of these two aspects. 
We suppose the distribution of $u_{ij}$ is known, but the actual value can only be observed by some feedback after the allocation is realized. Thus, even though we can roughly estimate the outcome of an algorithm through offline experiments, we still want to determine its actual performance through online experiments. On advertisement platforms, the feedback may be the actual click rates, which can be seen after advertisements are impressed to users. Meanwhile,
we suppose $c_{ij}$ is fully known. Usually, $c_{ij}$ is decided by some auction mechanisms such as the first price auctions or second price auctions. So it only depends on the bids from the buyers. Thus, it is known to the platform before we start allocation.
For simplicity, we also denote $\mathbf{b}=(b_1,\cdots,b_n),\mathbf{U}=(u_{ij})_{i,j},\mathbf{C}=(c_{ij})_{ij}$ as their matrix forms. 

Additionally, our model is applicable to e-commerce platforms such as eBay, where the quantity of a particular product can be considered as a form of generalized budget.  In this scenario, the ``buyer'' is the product and the ``item'' is the customer. $c_{ij}$ is the number of product $j$ purchased by customer $i$ and typically $c_{ij}=1$. The utilities in this case may be the revenue generated by the platform, usually in the form of agency fees, or they can reflect the price of the product when customers engage in bidding. Furthermore, our model can accommodate common linear constraints observed in practical settings, such as a limit on the number of items that can be purchased by a single buyer.

\subsection{Allocation Matrix and Objective}
Platforms may use an allocation algorithm between buyers and items. The allocation algorithm aims to allocate each item to a buyer while not allowing exceeding buyers' budgets.
To represent the realization of an allocation, we define an allocation matrix
$\mathbf{W}\in \{0,1\}^{m\times n}$, where $w_{ij}=1$ if the item $i$ is allocated to buyer $j$, otherwise $w_{ij}=0$. Since the item cannot be duplicated, $\sum_j w_{ij}\le 1$ for any $i$. When $\sum_j w_{ij}=0$, it means this item is aborted, i.e., it will not be allocated to any buyer. The allocation matrix is actually the adjacency matrix of the corresponding bipartite graph. We use $\mathbf{w}_i$ to represent the $i$-th row of $\mathbf{W}$. 

Imagine that the platform recently proposed an alternative allocation algorithm that is represented by $\mathbf{W}^1$, while the original allocation algorithm is represented by $\mathbf{W}^0$. The platform wants to assess if the newly proposed allocation outperforms the original one. 
However, it is challenging to feasibly impose two allocation approaches for the same pool of buyers or items. As the example in Section \ref{sec:intro} shows, with the existence of budget constraints, randomly assigning items to one of the allocations may lead some buyers to exceed their budgets; similarly, randomly assigning buyers to one of the allocations may lead one item being allocated to multiple users. 

The goal of the platform is to estimate the difference in the sum of utilities resulting from two allocations, or the so-called total treatment effect (TTE):
$$\tau=\sum_{i,j} u_{ij}w_{ij}^1-\sum_{i,j} u_{ij}w_{ij}^0$$

It's worth mentioning that we may want to relax the allocation matrix from $\{0,1\}^{m\times n}$ to $[0,1]^{m\times n}$. Imagine in an e-commerce platform customers can see multiple products in their recommendation lists and may click and buy them with a probability distribution based on some behavior models. Instead of using a one-hot vector $\mathbf{w}_i\in\{0,1\}^n$ for any item $i$, we replace it with a distribution vector $\mathbf{w}_i\in [0,1]^n$ such that $\sum_{j=1}^n w_{ij}\le 1$. In that situation, the cost is not simply the sum of $c_{ij}w_{ij}$, which calls for other variables such as the position of items in the list. In this paper, we mainly consider the one-hot setting.

\subsection{Experimental Design}
To estimate TTE, we need to perform a special experiment to collect data about $\mathbf{U}$.
Recall that the exact value of $u_{ij}$ can only be observed after an allocation is realized. Therefore, we cannot observe the whole matrix $\mathbf{U}$. Instead, we only observe at most one data point for each item, i.e. each row of $\mathbf{U}$. So we define the observation matrix to show all the values observed for estimation.
\begin{definition}[observation matrix]
When the realized allocation matrix is $\mathbf{W}$, the corresponding observation matrix is $\mathbf{O}$ where $o_{ij}=w_{ij}u_{ij}$ for any $i,j$.
\end{definition}

An experimental design is a distribution $p(\mathbf{W})$  over all allocations. Then we will sample an allocation $\mathbf{W}$ from the distribution and obtain the observation matrix. Now our goal becomes finding a well-performing distribution $p(\mathbf{W})$ to estimate TTE.

Without budget constraint, Bernoulli randomization on item level, which means we randomly allocate the item according to the old algorithm or new algorithm, can perfectly solve this problem. However, the budget constraints require that the cost of any buyer should not exceed their budget. Thus, not every allocation can be feasible and this will restrict distribution $p(\mathbf{W})$. Using the language of the graph, the sum of the edge weight of each node will not exceed its budget. We describe this condition as budget-satisfying allocation.
\begin{definition}[budget-satisfying allocation]
We define an 
allocation $\mathbf{W}$ is budget-satisfying if for any buyer $j$, $\sum_i c_{ij}w_{ij}\leq b_j$. The set of all budget-satisfying allocations is denoted by $\mathcal{S}$. 
\end{definition}

We can suppose that $\mathbf{W}^0,\mathbf{W}^1$ are both budget-satisfying.

To ensure any allocation $\mathbf{W}$ sampled from $p(\cdot)$ satisfies the budget constraint, the support of $p(\cdot)$ should be $\mathcal{S}$. However, enumerating $\mathcal{S}$ completely is very hard even when $c_{ij}$ are integers and can be seen as a knapsack count problem, which is a typical $\#P$ problem and can only be solved by approximation algorithms. Thus it is unrealistic to directly design the distribution over $\mathcal{S}$. Here we will consider a simple but similar set of $\mathcal{S}$ to avoid the enumeration.

In detail, we first design the allocation distribution without considering budget constraints. Then we modify our allocation matrix sampled from the distribution to control the budget of each buyer and get a budget-satisfying allocation. 

For simplicity, we only consider the allocation distribution where the allocation of each item is independent. Suppose we allocate each item $i$ to buyer $j$ with probability $x_{ij}$, the the unmodified distribution is $p(\mathbf{W})=\Pi_{i,j:w_{ij}=1}x_{ij}$. We can define an experiment matrix to represent such distribution.

\begin{definition}[Experiment matrix]
Suppose we have a distribution $p(\mathbf{W})=\Pi_{i,j:w_{ij}=1}x_{ij}$, the corresponding experiment matrix $\mathbf{X}\in [0,1]^{m\times n}$.
\end{definition}
We will abuse the notation to consider $\mathbf{X}$ as the distribution. Notice that if $\sum_{j} x_{ij}< 1$, then there is a chance that item $i$ is aborted.

\begin{example}[Bernoulli randomization]
In conventional A/B testing (i.e., Bernoulli randomization), we randomly choose between $\mathbf{w}_i^1$ and $\mathbf{w}_i^0$ as the real allocation vector of item $i$. That is, item $i$ is either allocated according to the controlled allocation with probability $p$ or the treatment allocation with probability $1-p$. The corresponding experiment matrix is:
$$\mathbf{X}=p\mathbf{W}^1+(1-p)\mathbf{W}^0.$$
\end{example}

\paragraph{Modification} After obtaining the unmodified distribution, there are chances that we may sample a not budget-satisfying allocation as Example \ref{example:overspend} illustrates. So we have to convert an unmodified distribution to a distribution with support $\mathcal{S}$, which is called modification.

\begin{example}[Example of overspending]
Suppose there are $2k$ items and 2 buyers. If $c_{ij}=1$ for any $i,j$ and $b_1=b_2=n$. $W^1_{ij}$ satisfies that
\begin{equation*}
 W^1_{ij}=\begin{cases}
1 & 1\le i\le k\ and\ j=1\    \\
1 & k+1\le i\le 2k\ and\ j=2  \\
0 & otherwise
\end{cases} 
\end{equation*}

and $W_{ij}^0=1-W_{ij}^1$ for any $i,j$.

So both $\mathbf{W}^1$ and $\mathbf{W}^0$ 
are
budget-satisfying. However, if we use Bernoulli randomization, e.g.,  $\mathbf{W}=\frac{1}{2}(\mathbf{W}^1+\mathbf{W}^0)$. Only with probability $\frac{\binom{k}{2k}}{2^{2k}}$ the allocation matrix sampled from $\mathbf{W}$ is budget-satisfying, which happens only when each buyer is assigned with half of the items. When $k\rightarrow\infty$, for almost surely there exists one buyer who will exceed his budget.
\label{example:overspend}
\end{example}

It is natural that we can resample the allocation until we get a budget-satisfying allocation. But as Example \ref{example:overspend} shows, this probability can be negligible hence a waste of computational time. So we will not assign some of the items to control the cost of buyers, which is also called \textit{throttling} in some auction mechanisms. In math, it is a modification function $M(\mathbf{W})=\mathbf{W'}$, converting any allocation matrix $\mathbf{W}$ to a budget-satisfying allocation matrix $\mathbf{W'}$. There are several ways of throttling. We mainly consider sequential throttling and random throttling.

In sequential throttling, we just keep the smaller numbered item. In other words, we will throttle items from item $m$ to item $1$ until no buyer exceeds his budget. In random throttling, we first uniformly sample a permutation of the items and then use sequential throttling based on the new index. Example \ref{example:seq} and Example \ref{example:random} give the formal definition.

\begin{definition}[Sequential throttling]
Suppose the unmodified allocation matrix is $\mathbf{W}$. Then we define $i^*(j)$ is the largest index for buyer $j$ such that $\sum_{i=1}^{i^*} w_{ij}c_{ij}\le b_j$. We modify the matrix by $x'_{ij}=x_{ij}\vmathbb{1}(i\le i^*(j))$ where $\vmathbb{1}$ is the indicator function. The modified matrix is guaranteed to be budget-satisfying.
\label{example:seq}
\end{definition}

\begin{definition}[Random throttling]
Suppose the unmodified allocation matrix is $\mathbf{W}$. We sample a permutation $\Sigma(i)$ of $[m]$ uniformly and randomly. If $i^*(j)$ is the largest index for buyer $j$ such that $\sum_{\Sigma(i)\le i^*(j)} w_{ij}c_{ij}\le b_j$. We modify the matrix by $x'_{ij}=x_{ij}\vmathbb{1}(\Sigma(i)\le i^*(j))$.
\label{example:random}
\end{definition}

In Example \ref{example:overspend}, suppose now the buyer 1 is assigned with item $1,2,\cdots,k+1$. Since his budget is $k$, we have to choose one item and not assign it to buyer 1. In sequential throttling, we always choose item $k+1$ and in random throttling, we uniformly choose among item $1,2,\cdots,k+1$. Intuitively, the sequential throttling will be more unbalanced since the items at the end will hardly be allocated, while the random throttling can ensure the allocated probability of each item is similar and not too small. However, in the online setting, sequential throttling is more reasonable, since we cannot withdraw the previous allocation after the budget is exhausted. Whatever method we use, the modified allocation will be budget-satisfying. We can realize the final allocation matrix and obtain the observations.

\subsection{Estimator}
Next, we discuss how we construct the estimator for TTE. After the outcome $\mathbf{O}$ observed, we assume that $p_{ij}=\Pr_{\mathbf{W}\sim\mathbf{X}}(o_{ij}=u_{ij})=\Pr_{\mathbf{W}\sim\mathbf{X}}\left(M(\mathbf{W})_{ij}=1\right)$. It represents the actual probability that item $i$ is allocated to buyer $j$ under the modification function $M$ and experiment matrix $\mathbf{X}$. By the Horvitz–Thompson estimator an ideal estimator is 

$$\bar{\tau}(\mathbf{O})=\sum_{i,j} \frac{o_{ij}w^1_{ij}}{p_{ij}}-\sum_{i,j} \frac{o_{ij}w^0_{ij}}{p_{ij}}.$$ 

To ensure the estimator is well-defined, we define $0/0=0$. Notice that when $p_{ij}=0$, $o_{ij}$ is always $0$ too. To estimate $\tau$, we only need to pay attention to item-buyer pairs such that $w^1_{ij}>0$ and $w^0_{ij}>0$. The following proposition illustrates that as long as we put non-zero probability in those important pairs, then $\bar{\tau}$ is unbiased. We state the following proposition and defer the proof to Appendix A.

\begin{proposition}
If for any $i,j$ such that $w_{ij}^1+w_{ij}^0>0$, $x_{ij}>0$. Then $\mathbb{E}_{\mathbf{W}\sim \mathbf{X}}[\bar{\tau}(\mathbf{O})]=\tau$.
\label{prop:HTestimator}
\end{proposition}
In practice, The Horvitz–Thompson estimator is often limited by excessive variance. So it is wise to use the Hajek estimator to obtain less variance at the expense of a small bias \cite{khan2021adaptive,eckles2017design}. \footnote{The formula of Hajek estimator in our setting is $$m\left(\sum_{i,j} \frac{o_{ij}w^1_{ij}}{p_{ij}u_{ij}}\right)^{-1}\sum_{i,j} \frac{o_{ij}w^1_{ij}}{p_{ij}}-m\left(\frac{o_{ij}w^0_{ij}}{p_{ij}u_{ij}}\right)^{-1}\sum_{i,j} \frac{o_{ij}w^0_{ij}}{p_{ij}}.$$}

However, $p_{ij}$ is difficult to compute since it relies on not only the probability of overspending but also the way of throttling. We can approximate the probability by replication, a method based on Monte Carlo \cite{fattorini2006applying, aronow2013estimating}, but it's not always feasible due to the computation time. So we provide another simple alternative.  Intuitively, when item $i$ is not throttled by the throttling function with high probability, then $p_{ij}$ is close to $x_{ij}$. We replace $p_{ij}$ with $x_{ij}$:

$$\hat{\tau}(\mathbf{O})=\sum_{i,j} \frac{o_{ij}w^1_{ij}}{x_{ij}}-\sum_{i,j} \frac{o_{ij}w^0_{ij}}{x_{ij}}.$$

We can also create a Hajek-like estimator by replacing $p_{ij}$ with $x_{ij}$ in the formula. However, we do not discuss it in this paper. Figure \ref{fig:flow1} summarizes our model. In Section \ref{sec:bias} we just use present tense analysis of the bias and variance of $\hat{\tau}$.

\begin{figure}[!ht]\centering
  \includegraphics[scale=0.3]{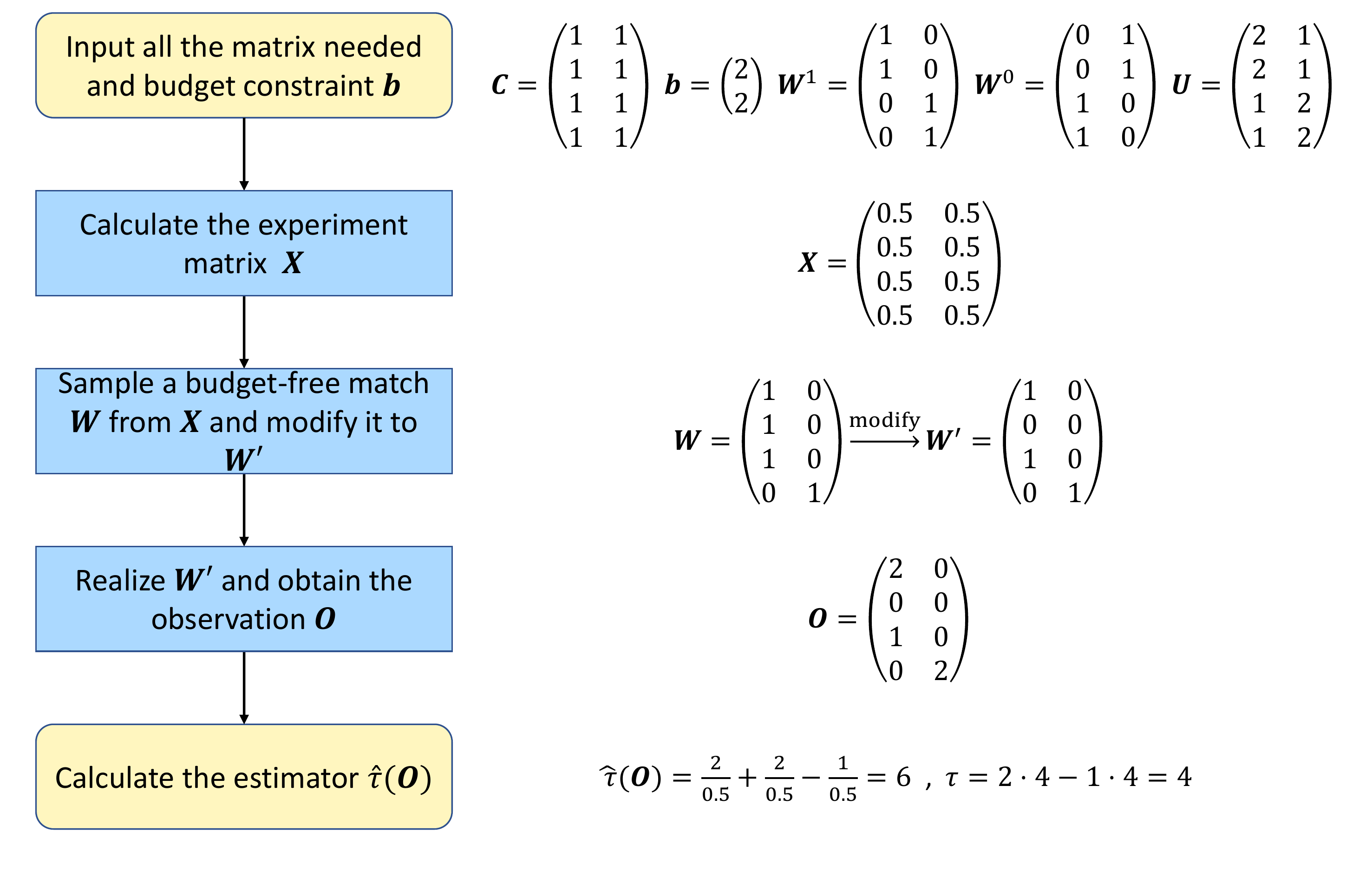}
  \caption{\textbf{The flow of experiment model.}}
  \label{fig:flow1}
\end{figure}

\section{Analysis of Bias and Variance}
\label{sec:bias}
In this section, we 
provide an analysis for the bias and variance of $\hat{\tau}(\mathbf{O})$.

\subsection{Bias}
First, we show that when the budget is sufficient, $\hat{\tau}$ is unbiased, which echoes prior studies such as 
 \citet{liu2021trustworthy}.
 Intuitively, no item will be throttled with a sufficient budget, and thus there is no bias brought by budget constraints. When there is more budget, bias becomes less. Formally, we show the following proposition to describe the relationship between budget and bias.

\begin{proposition}
Given an experimental matrix $\mathbf{X}$, when the budget $b_j\ge \sum_i c_{ij}\vmathbb{1}(x_{ij}>0)$ for any buyer $j$ and $x_{ij}>0$ for any pair $i,j$ satisfies $w^0_{ij}+w^1_{ij}>0$, then $\mathbb{E}[\hat{\tau}(\mathbf{O})]=\tau$. Particularly, if we choose $\mathbf{X}=p\mathbf{W}^1+(1-p)\mathbf{W}^0$,  when $b_j\ge \sum_{i} c_{ij}(w_{ij}^0+w_{ij}^1)$ for any $j$, $\hat{\tau}(\mathbf{O})$ is unbiased. 
\label{prop:unbias}
\end{proposition}

\begin{proof}
The largest cost of $j$ is $\sum_i c_{ij}\vmathbb{1}(x_{ij}>0)$ for any $j$, which occurs when every potential possible item is allocated to $j$. So when $b_j\ge \sum_i c_{ij}\vmathbb{1}(x_{ij}>0)$, it is impossible to overspend. Then $\Pr_{\mathbf{W}\sim\mathbf{X}}(i\ is\ allocated\ to\ j)=x_{ij}$, i.e. $\hat{\tau}(\mathbf{O})=\bar{\tau}(\mathbf{O})$.
\end{proof}

Although in other general situations $\hat{\tau}$ is biased, the bias can be bounded under some reasonable assumptions. There are several insights to control the bias. According to the law of large numbers, with a high probability, the cost of each buyer is close to their expected cost, especially when there are lots of items. So we only need to control the expected cost of each buyer. Then only the items with large indexes will be throttled and each item is assigned a large index with a small probability after random permutation. So the probability of throttling is small under random throttling.

As we discussed, we define the special family of experiments as the expected budget-satisfying experiment matrix.

\begin{definition}[Expected budget-satisfying experiment matrix]

We say an experiment matrix $\mathbf{X}$ is expected budget-satisfying if for any $j$, $\sum_{i} x_{ij}c_{ij}\le b_j$.
\end{definition}

It's directly derived that all expected budget-satisfying experiments form a convex set. And we get the corollary that any Bernoulli randomization is a budget-satisfying experiment design.

Now we state our main theorem.

\begin{theorem}
If the cost $c_{ij}\in[l,h]$ for any $i,j$. Then for any expected budget-satisfying experiment matrix $\mathbf{X}$ which satisfies $x_{ij}\ge x_0$ for any $i,j$ satisfies $w^0_{ij}+w^1_{ij}>0$. If we use random throttling and the item allocated to any buyer goes to infinity in allocation 1 or allocation 0, the average bias goes to 0. That is, $\frac{\sum_{i}(w_{ij}^1+w_{ij}^0)}{n}\rightarrow\infty$ for any $j$. Then $\frac{1}{m}\left|\mathbb{E}[\hat{\tau}(\mathbf{O})]-\tau\right|\rightarrow 0$. To be specific, $\frac{1}{m}\left|\mathbb{E}[\hat{\tau}(\mathbf{O})]-\tau\right|\le O(m_{k})^{-1/3})$ where $m_k=\min_j \sum_{i}(w_{ij}^1+w_{ij}^0)$.
\label{thm:unbias}
\end{theorem}

\subsection{Variance}
As for the variance, we start from the following proposition, which calculates variance in the situation with a sufficient budget. The calculation details are placed in Appendix A.

\begin{proposition}
When the budget is sufficient, i.e. satisfies the condition in Proposition \ref{prop:unbias}. Suppose $\mathbb{E}[u_{ij}]=\mu_{ij},\mathrm{Var}(u_{ij})=\sigma_{ij}^2$. The variance of our estimator is

\begin{footnotesize}
 $$\mathrm{Var}(\hat{\tau}(\mathbf{O}))=\sum_{i,j} \frac{(\mu_{ij}^2+\sigma_{ij}^2)(w_{ij}^1+w_{ij}^0)}{x_{ij}}-\sum_{i,j}\mu_{ij}^2(w_{ij}^1+w_{ij}^0)+2\sum_{i,j,j'}\mu_{ij}w_{ij}^1\mu_{ij'}w_{ij'}^0.$$    
\end{footnotesize}
\label{prop:varfullbudget}
\end{proposition}

When the budget is insufficient, due to the bias of $\hat{\tau}$, we calculate the mean square error (MSE) instead of variance. The accurate value of MSE relies on the probability of overspending, which is intractable. So we only give an upper bound for it. The proof is also deferred to Appendix A.

\begin{proposition}
\begin{small} The MSE of $\hat{\tau}(\mathbf{O})$ is bounded by
$$MSE(\hat{\tau}(\mathbf{O}))\le \sum_{i,j}\frac{(\mu_{ij}^2+\sigma_{ij}^2)(w_{ij}^1+w_{ij}^0)}{x_{ij}}+\left(\sum_{i,j}\mu_{ij}w_{ij}^1\right)^2+\left(\sum_{i,j}\mu_{ij}w_{ij}^0\right)^2.$$
\end{small}
\label{prop:varlimitbudget}
\end{proposition}

\section{Optimization}
\label{sec:opt}
In this section, we will discuss how to obtain a well-performing experimental design in both offline and online settings.
\subsection{Offline Setting}
\label{subsec:off}
From the calculation in the previous section, the only variable in the upper bound of MSE is $\mathbf{X}$. So if we minimize $\sum_{i=1}^n \frac{(\mu_{ij}^2+\sigma_{ij}^2)(w_{ij}^1+w_{ij}^0)}{x_{ij}}$, then we can control the MSE of $\hat{\tau}(\mathbf{O})$. To give an effective experiment, we choose the solution $\mathbf{X}_1^*$ of the following optimization problem:

\begin{equation}
\begin{aligned}
    \min_{x_{ij}} \quad& \sum_{i,j} \frac{(\mu_{ij}^2+\sigma_{ij}^2)(w_{ij}^1+w_{ij}^0)}{x_{ij}}\\
    \mbox{s.t.} \quad & 0< x_{ij}\le 1 &\forall i,j\\
    & \sum_{j} x_{ij}\le 1 & \forall i \\
    & \sum_{i} x_{ij}c_{ij}\le b_j & \forall j\\
    \label{formula:opt1}
\end{aligned}
\end{equation}

We do not allow $x_{ij}=0$ here since they are the denominators. However, when $w_{ij}^1+w_{ij}^0=0$ for pair $i,j$, 
$x_{ij}$ will be $0$ in the optimal solution. In the numerical calculation, we will calculate a solution with an arbitrarily small value for those pairs. Thus it will approach the optimal solution with arbitrary accuracy. Without the linear budget constraint, we can solve the problem directly using Lagrangian. If we define another optimization problem:

\begin{equation}
\begin{aligned}
    \min_{x_i} \quad& \sum_{i,j} \frac{(\mu_{ij}^2+\sigma_{ij}^2)(w_{ij}^1+w_{ij}^0)}{x_{ij}}\\
    \mbox{s.t.} \quad & 0< x_{ij}\le 1 &\forall i,j\\
    & \sum_{j} x_{ij}\le 1 & \forall i \\
    \label{formula:opt2}
\end{aligned}
\end{equation}
The solution $\mathbf{X}_2^*$ of Problem \ref{formula:opt2} is 
$$x_{ij}^*=\frac{(w_{ij}^1+w_{ij}^0)\sqrt{\mu_{ij}^2+\sigma_{ij}^2}}{\sum_{j'} (w_{ij'}^1+w_{ij'}^0)\sqrt{\mu_{ij'}^2+\sigma_{ij'}^2}}.$$ 
When the budget is more, the solution of Problem \ref{formula:opt1} is closer to Problem \ref{formula:opt2}. So this solution provides a simple approximation for the optimal experiment matrix. In addition, if we assume $u_{ij}$ is sampled from the same distribution independently for any pair $i,j$. Then Bernoulli randomization with $p=1/2$ is just the optimal solution.

As for the precise solution of Problem \ref{formula:opt1}, due to the convex objective function with $\mathbf{X}$ and the linear constraint, we can use some convex optimizer to solve it. 

\subsection{Online Setting}
In practical applications, it is not always feasible to have complete knowledge of all items beforehand and determine the allocation accordingly. For instance, in advertisement platforms, user search requests arrive in a sequential manner and the platform must immediately decide which advertisement to display on the web page. Similarly, in ride-hailing platforms, a driver must be assigned to a customer as soon as their request is received. To address these real-world scenarios, we extend our basic offline model to an online model. This extension accounts for the dynamic nature of the problem, where items arrive in streams, and the allocation must be made based on the current information.

Still, we assume there is an item sequence that item $i$ will bring utility $u_{ij}$ and cost $c_{ij}$ for buyer $j$. We say the history information up to item $i$ contains the cost and the real revenue observed up to $i$. That is, $H_i=(\mathbf{C}_i,\mathbf{O}_i)\cup H_{i-1}$ and $H_0=\emptyset$. The remaining budget can be calculated by $H_i$. An experiment design is a matrix $\mathbf{X}$ where each line $X_i=p(H_{i-1},\mathbf{b})$ relies on the history information and budgets.  

Notice that though in the online setting we cannot even get the whole allocations $\mathbf{W^1},\mathbf{W^0}$. Instead we know $\mathbf{W}_i^1,\mathbf{W}_i^1$ at the arrival of item $i$. We suppose $\mathbf{W}^1_i,\mathbf{W}^0_i$ will not change with the history information. We summarize the online model by Figure \ref{fig:flow2}.

\begin{figure}[!ht]\centering
  \includegraphics[scale=0.25]{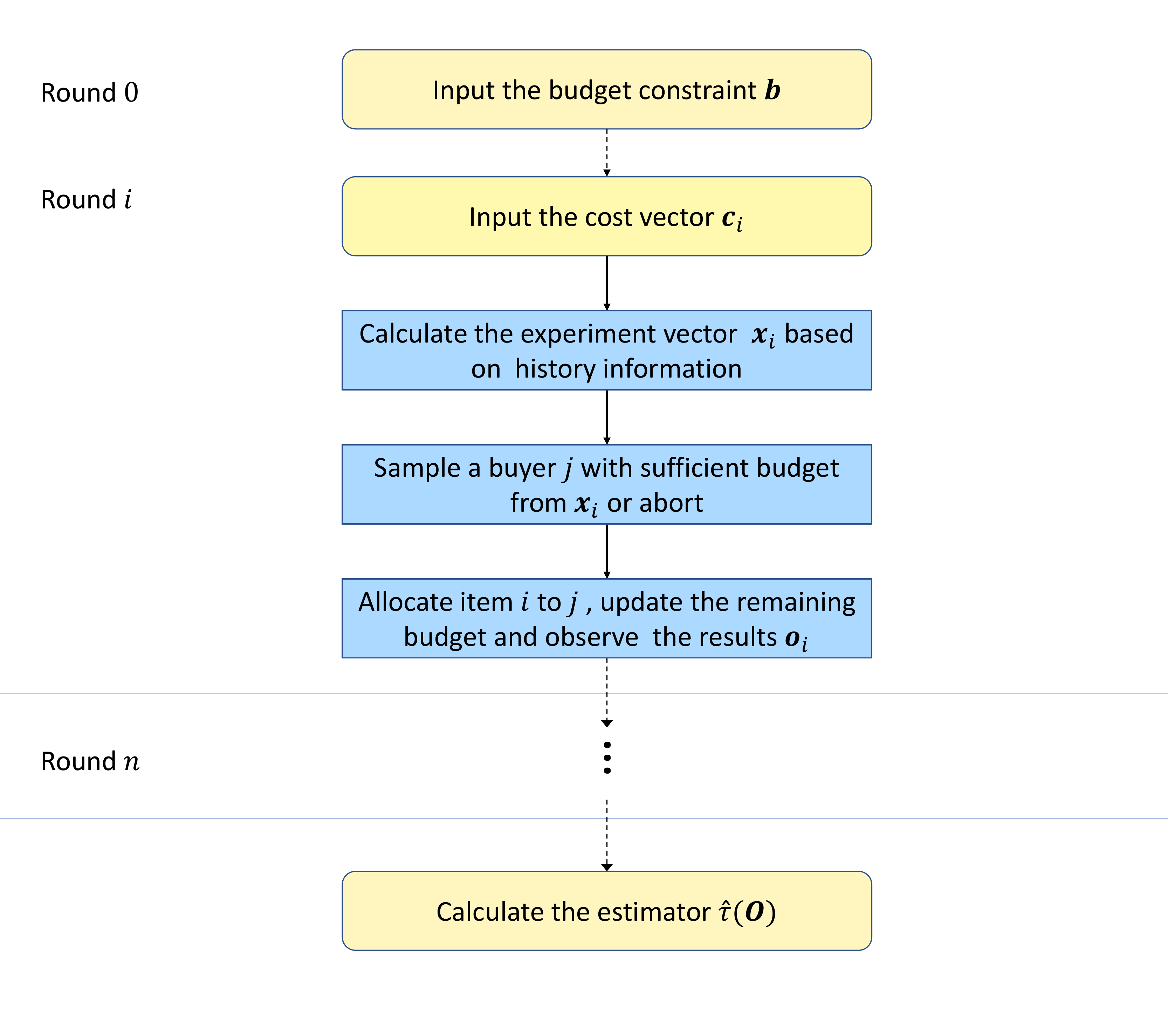}
  \caption{\textbf{The flow of online experiment model.}}
  \label{fig:flow2}
\end{figure}

Similar to the objective function in Section \ref{subsec:off}, we aim to minimize
$$\sum_{x_{ij}>0} \frac{(\mu_{ij}^2+\sigma_{ij}^2)(w_{ij}^1+w_{ij}^0)}{x_{ij}}.$$

Following the algorithm in \cite{agrawal2014dynamic,zhang2022online}, we use Algorithm \ref{alg:online} to solve the online problem.
The algorithm does not use the history information such as the history allocation or the left budget. In brief, we calculate the optimal solution up to item $i$ based on the proportional split budget. And we use the $i$-th row of the temporary optimal solution as the allocation vector of item $i$.

\begin{algorithm}
  \SetAlgoLined
  \KwIn{$\mathbf{b}$}
  \KwOut{an estimation $\hat{\tau}$}

  $\mathbf{O}\leftarrow\mathbf{0}$\tcp{observation matrix}
  $\mathbf{X}\leftarrow\mathbf{0}$\tcp{experiment matrix}
  $\mathbf{b'}\leftarrow\mathbf{0}$\tcp{current used budget}
  \For{$i\leftarrow 1$ \KwTo $m$}{
    add new request to $\mathbf{W}^0,\mathbf{W}^1$
    $\mathbf{X}\leftarrow Solve(\mathbf{W}^0,\mathbf{W}^1,i*\mathbf{b}/m)$\tcp{solve the optimization problem by request up to time $i$ under scaled budget}
    $j\leftarrow sample(\mathbf{X}_i)$\tcp{sample the allocation index of item $i$}
    \If{$c_{ij}+b'_j\le b_j$\tcp{feasible test}}  
    {  
        \tcp{allocate $i$ to buyer $j$ and observe the results}
        $b_j'\leftarrow b_j'+c_{ij}$\;
        $O_{ij}\leftarrow observe(i,j)$\;
    }
  }
  $\hat{\tau}\leftarrow estimator(\mathbf{O},\mathbf{X})$\tcp{estimate TTE}

  \caption{Online optimal experiment design}
  \label{alg:online}
\end{algorithm}

Under the ``random order'' assumption \cite{zhang2022online}, which means there is a pre-determined set of items and they come in random order, the algorithm performs well. Since we have to calculate a sub-problem with the size of $i$ at round $i$. If we suppose the solver takes time $O(f(m,n))$, the time complexity is $O(mf(m,n))$ and is still polynomial.

\section{Empirical Experiment}

In this section, we use some numerical experiments to verify our theoretical conclusions in both synthetic data and real-world data. We use the package \texttt{cvxpy} \cite{diamond2016cvxpy} in Python to solve the problem and choose the ``ECOS'' solver.

\subsection{Synthetic Data}
In the synthetic data, we mainly adjust two parameters: the supply-demand rate $r_1=\frac{m}{n}$ and the budget-cost rate $r_2=\frac{\sum_i b_i}{\max_{k=1,2}\sum_{i,j} c_{ij}w_{ij}^k}$. We set $n=10$ for example and set $m$ according to the parameter $r_1$ to show our results. We also test in different $n$ and find the results are robust. We uniformly and randomly sample $\mathbf{W}^1,\mathbf{W}^0$ from the set $\{\mathbf{e}_i\}$ where $\mathbf{e}_i$ is a one-hot vector and the value in its $i$-th dimension equals 1. Then $c_{ij}$ and $u_{ij}$ are sampled from the logarithmic normal distribution where $\mu=0,\sigma=1/4$ respectively. To obtain a non-zero TTE $\tau$, we double the $u_{ij}$ when $\mathbf{W}^1_{ij}=1$. It means that $\mathbf{W}^1$ allocates items to more proper buyers and brings more utility. The budget $b_i$ is simply the maximum cost used by buyer $i$ according to $\mathbf{W}^1$ or $\mathbf{W}^0$ times $r_2$. 

To ensure robust estimation of bias and variance, we sample 100 sets of parameters $\mathbf{W}_i^0,\mathbf{W}^1,\mathbf{U},\mathbf{C}$ and run 100,000 trials for each set. The following figures show the average results of 100 sets of parameters. As conventional Bernoulli randomization is infeasible under budget constraints, we compare our results to a Bernoulli randomization with item-end throttling. In this section, we refer to the modified Bernoulli randomization with throttling when we say Bernoulli randomization. In the offline setting, random throttling is employed, while in the online setting, sequential throttling is used.

To validate the result stated in Theorem \ref{thm:unbias} that the average bias approaches zero as the supply-demand rate $r_1$ increases for expected budget-satisfying experiments, we calculate the bias for both Bernoulli randomization and the optimal solution $\mathbf{X}_1^*$ under different $r_1$. We use the same instance in each experiment and divide it based on different values of $r_1$. The results, shown in Figure \ref{fig:supply}, indicate that both lines decrease and converge to zero as $r_1$ increases from 1 to 30. This suggests that when the number of items significantly exceeds the number of buyers on a platform, the bias can be disregarded.

Furthermore, to demonstrate the advantage of the optimal experimental design over Bernoulli randomization, we also compare their standard deviations in Figure \ref{fig:supply}. As depicted in the figure, when $r_1$ is low, the difference between them is relatively small. However, as $r_1$ increases, the optimal experiment design roughly reduces $20\%$ of the standard deviation. Hence, in large platforms, our optimal experiment design exhibits greater benefits.

To verify our claim in Section \ref{subsec:off} that the solution without budget constraints $\mathbf{X}_2^*$ is a good approximation for the solution with budget constraints $\mathbf{X}_1^*$, we compare their bias and variance across different values for the budget-cost rate $r_2$. As shown in Figure \ref{fig:budget}, their bias is comparable and approaches zero as $r_2$ increases, which aligns with Proposition \ref{prop:unbias}. The standard deviation of the two solutions differs by a constant factor and is subject to fluctuations. This discrepancy is likely due to the error in sampling and has no significant statistical significance.  This suggests that $\mathbf{X}_2^*$ can be the alternative of $\mathbf{X}_1^*$ to save computation time in some situations.

In order to evaluate the performance of our online algorithm, we compare its bias and standard deviation to the offline solution across different $r_1$. The results, shown in Figure \ref{fig:online}, reveal that the variance of our online algorithm has almost the same variance regardless of the value of $r_1$. Its bias is slightly lower as it adopts a more conservative approach when allocating the budget at the beginning. However, the difference is not significant. Overall our online algorithm is suitable in this setting, where each item's costs small are sampled randomly from a distribution.

In practice, there are chances that the assignments of some items are the same in the new algorithm and the old algorithm. At that time, regardless of the allocation, the contribution of these items in the TTE and our estimator is always zero. Thus the bias they brought is also zero. Based on this observation, we define a new parameter, the consistency rate $r_3=\frac{\sum_{i=1}^m \vmathbb{1}(\mathbf{w}^1_i=\mathbf{w}^0_i)}{m}$. It describes the ratio of items such that the allocations in $\mathbf{W}^1$ and $\mathbf{W}^0$ are the same. The more similar the two algorithms are, the higher the consistency rate is. To verify our expectation that higher $r_3$ brings less bias, we test the bias among different $r_3$ from $0$ to $1$ by setting $\mathbf{w}_i^1=\mathbf{w}_i^0$ for $i$ from 1 to $\lceil mr_3\rceil$. We also adjust $r_2$ from $1$ to $1.9$ for robustness, and the results are in Figure \ref{fig:consistent}. When consistency rate $r_3$ increases, the bias reduces for almost every budget-cost rate $r_1$. Thus the bias is not significant when the algorithms we compare are similar.

\subsection{Real-world Data}
To further test the performance of our experimental design, we conducted experiments using real-world data collected from the Tencent advertisement platform. Each item includes thousands of impression exposure to display the advertisements to users from WeChat or Tencent QQ. The platform runs multiple experiments simultaneously every day so a single experiment will only use part of the whole impressions and the budget is the maximum costs advertisers want to pay every day. The data we use comes from a single experiment, including basic information about around 14,000 items and 6000 buyers. Since we cannot get the real utility for every item-buyer pair, we stimulate the utility by the eCPM (effective cost per mille) module in the platforms. Considering the $O(mf(m,n))$ time complexity of our online algorithm, we use $\mathbf{X}_2^*$ as the ``optimal'' solution instead of directly applying Algorithm \ref{alg:online}. 

We sample 10,000 times for both Bernoulli randomization and optimal experiment. Given the large size of the real data, we present the relative bias and relative standard deviation, i.e. the ratios of the bias and standard deviation to the TTE, in Figure \ref{fig:real}. Our optimal experiment reduces the variance by sacrificing some bias. However, the standard deviation is much larger compared to bias, leading to a lower mean square error and making our experiment more effective in real-world scenarios.

\begin{figure}[htbp]
 \centering
  \begin{minipage}{0.48\linewidth}
   \centering
   \includegraphics[scale=0.25]{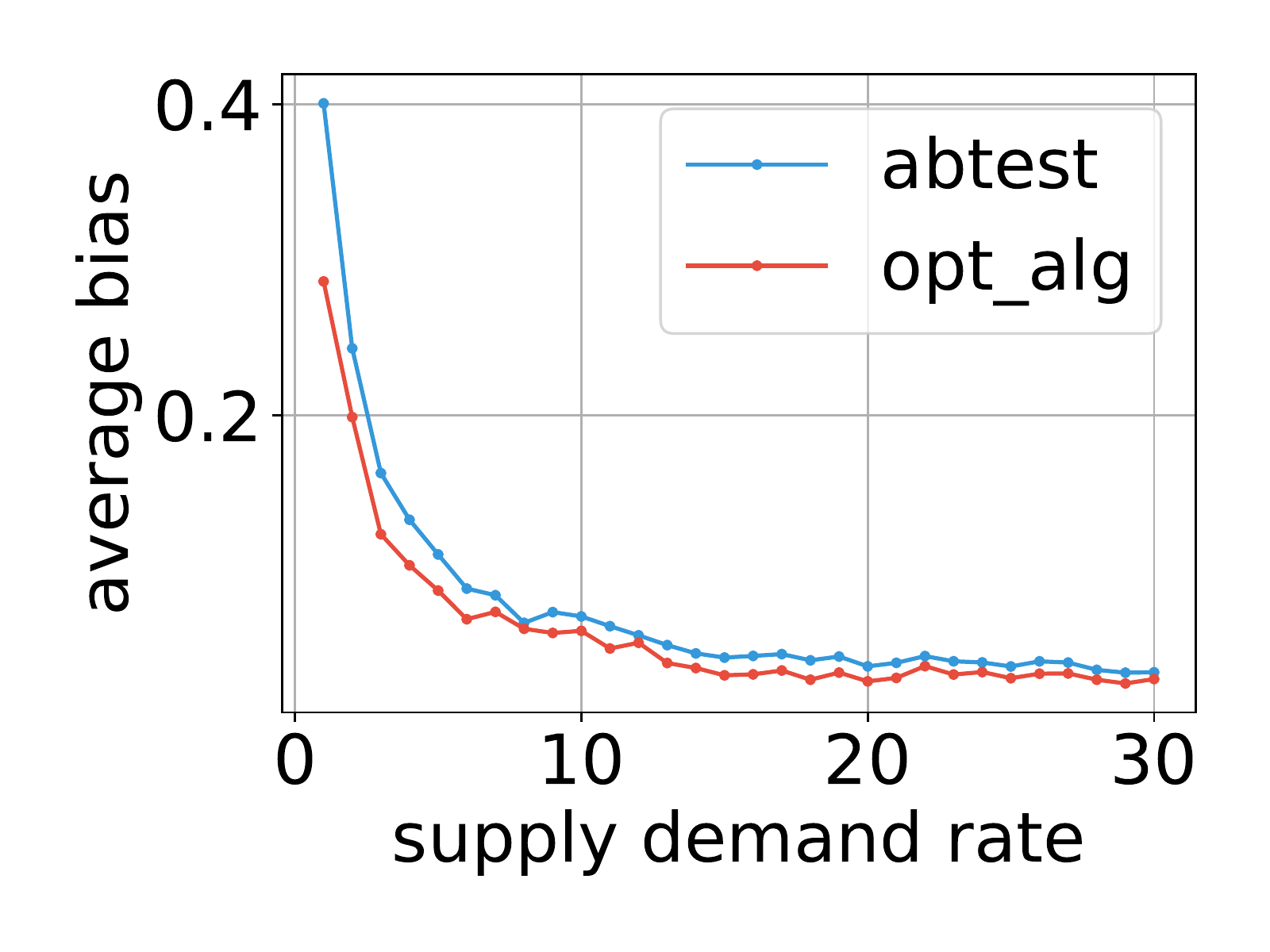}
  \end{minipage}
     \begin{minipage}{0.48\linewidth}
      \centering
      \includegraphics[scale=0.25]{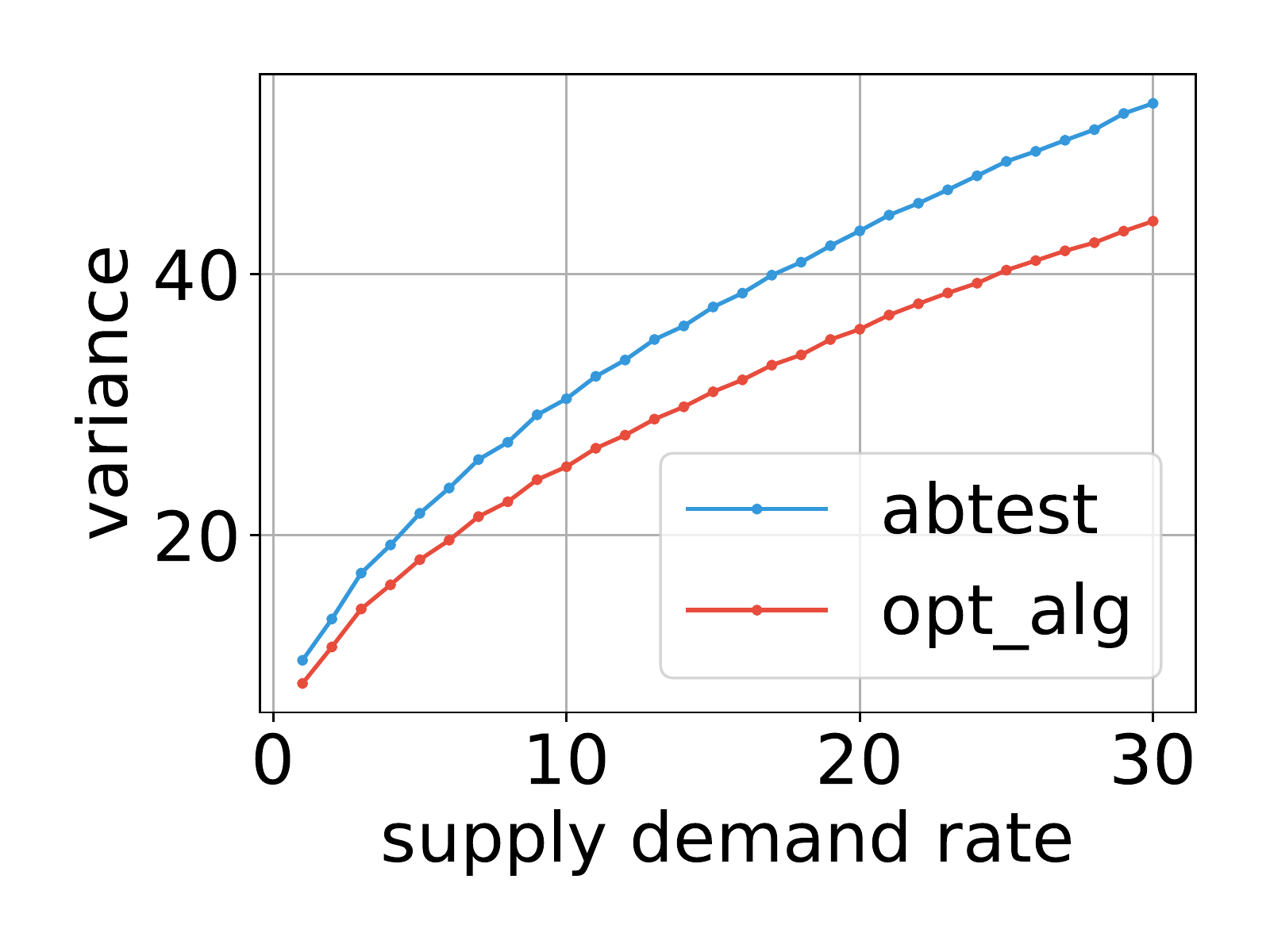}
     \end{minipage}

    \caption{\textbf{The bias and variance of optimal experiment and Bernoulli randomization where $p=0.5$ related to the supply-demand rate $r_1=m/n$. We fix $r_2=1$.}}
\label{fig:supply}
\end{figure}

\begin{figure}[htbp]
 \centering
  \begin{minipage}{0.48\linewidth}
   \centering
   \includegraphics[scale=0.25]{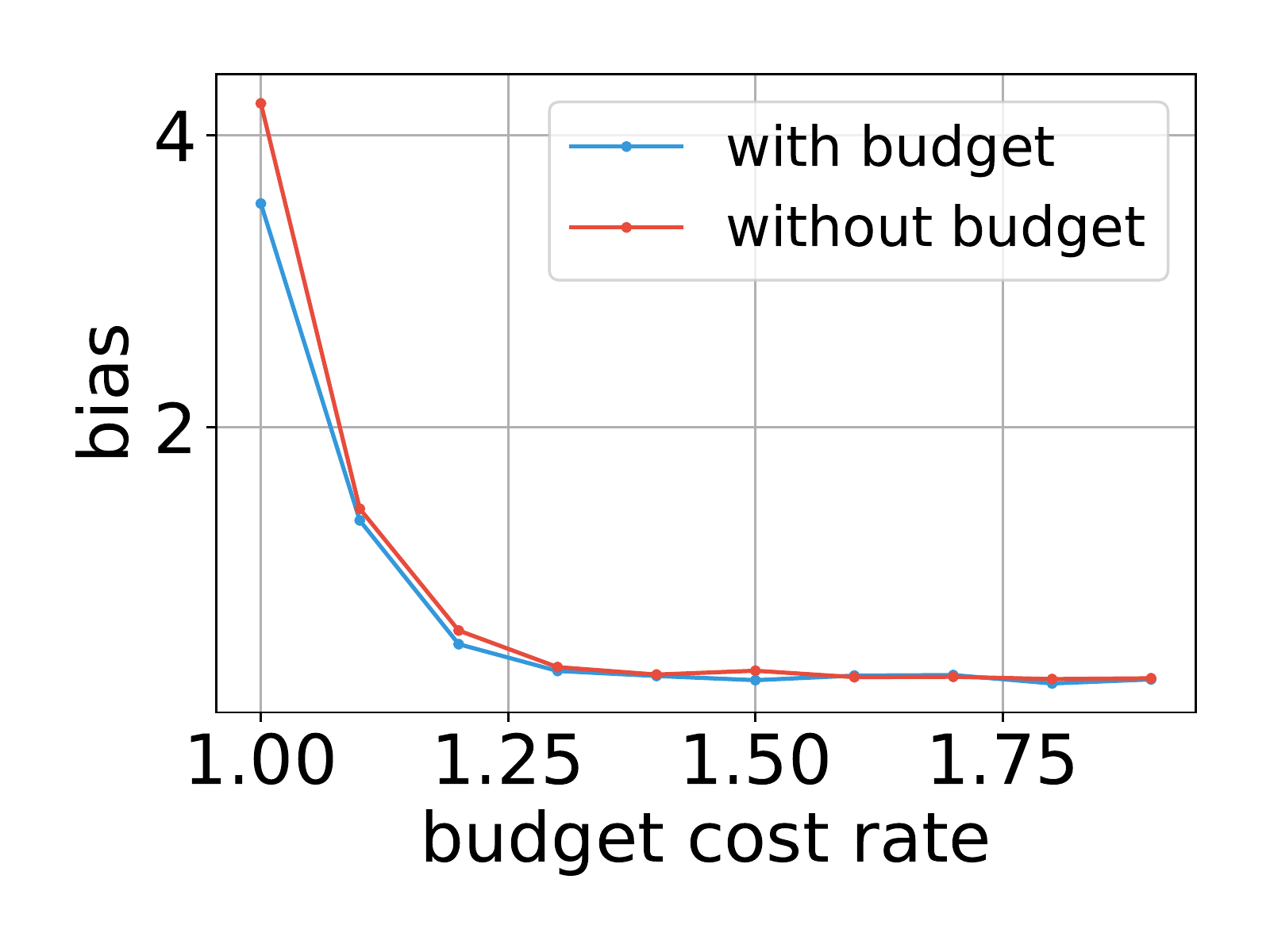}
  \end{minipage}
     \begin{minipage}{0.48\linewidth}
      \centering
      \includegraphics[scale=0.25]{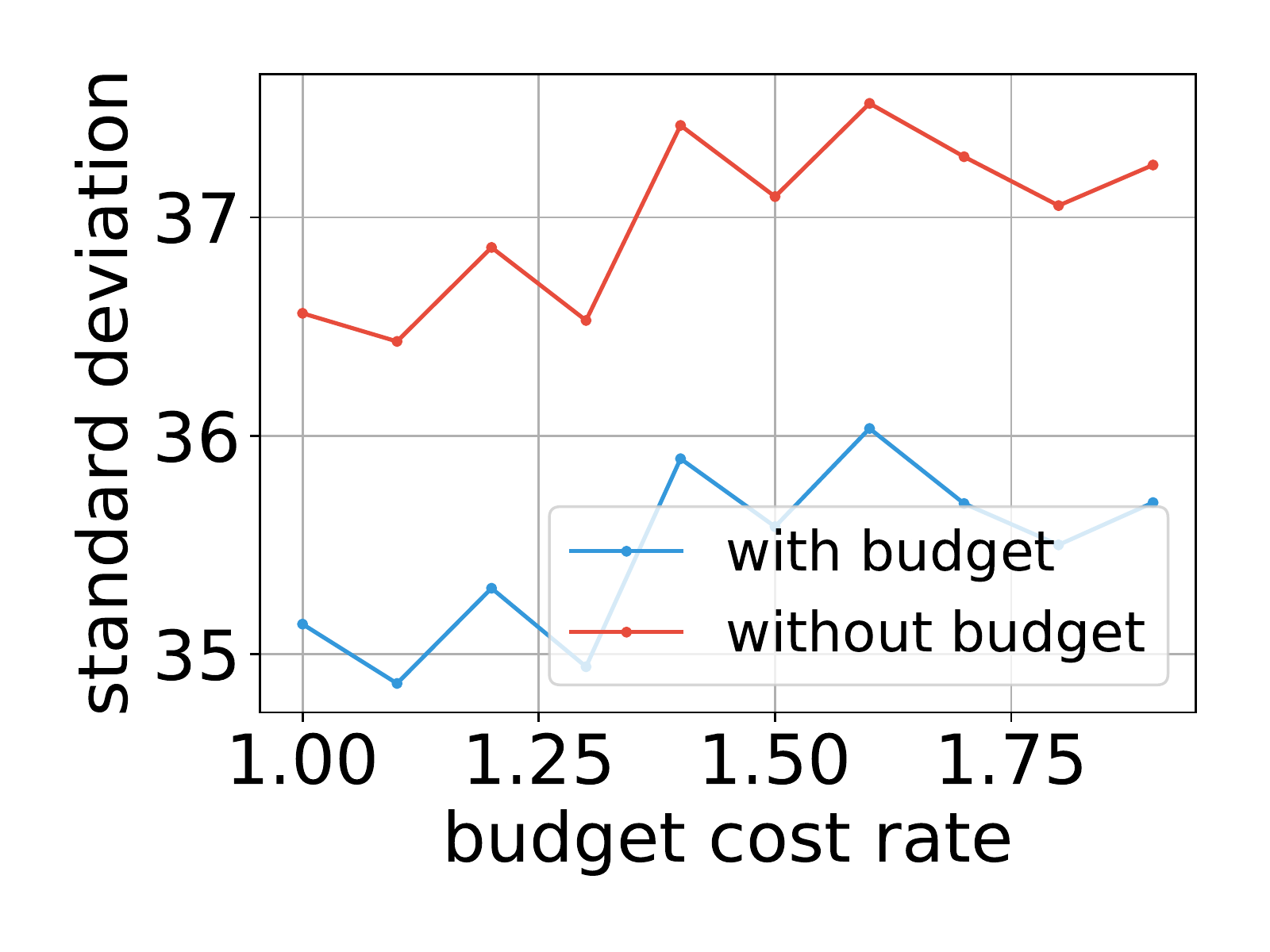}
     \end{minipage}
    \caption{\textbf{The bias and standard deviation of optimal experiment with budget constraint and without budget constraint related to budget-cost rate $r_2=\frac{\sum_i b_i}{\max_{k=1,2}\sum_{i,j} c_{ij}w_{ij}^k}$. We fix $r_1=20$.}}
\label{fig:budget}
\end{figure}

\begin{figure}[htbp]
 \centering
  \begin{minipage}{0.48\linewidth}
   \centering
   \includegraphics[scale=0.25]{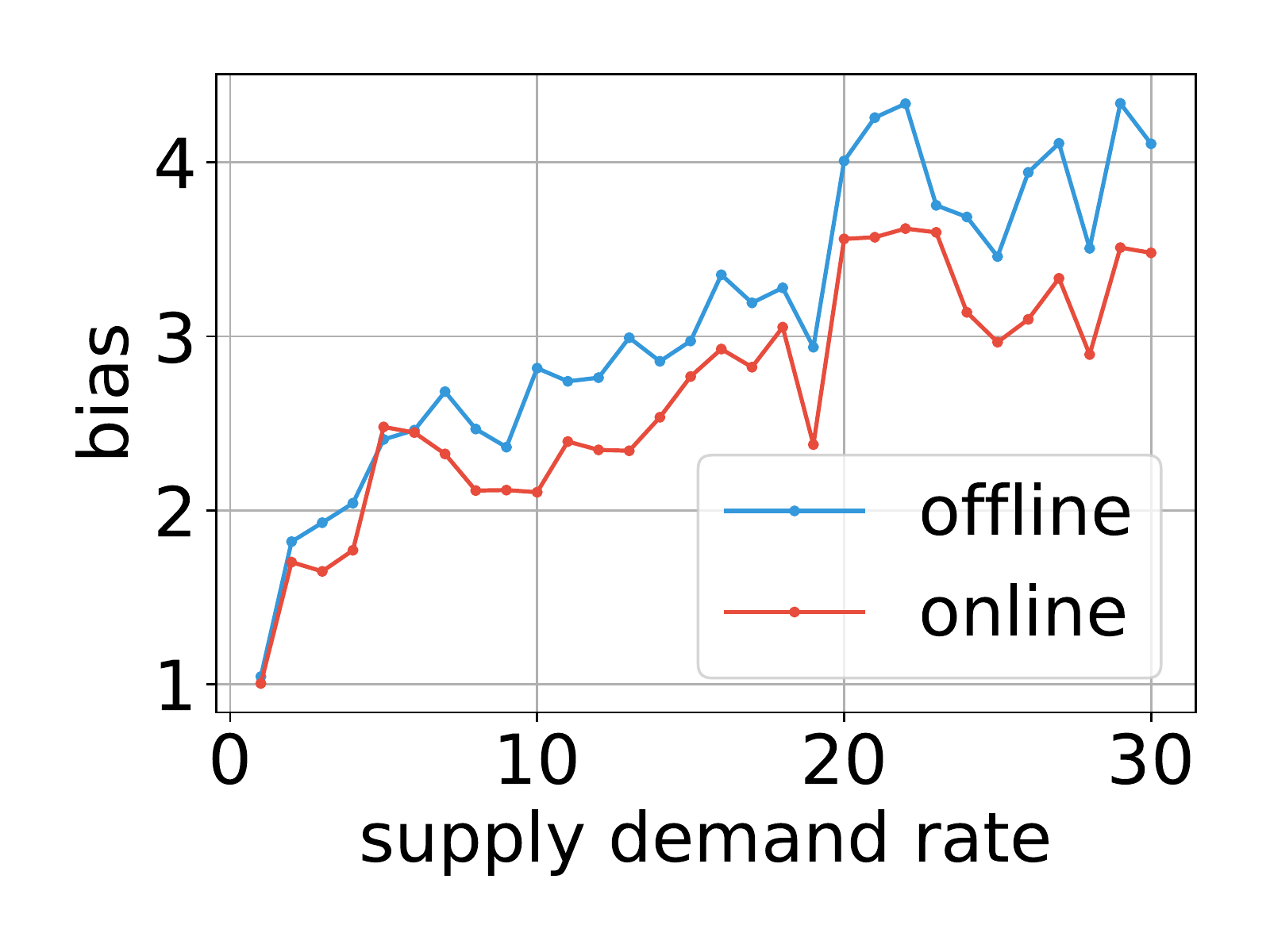}
  \end{minipage}
     \begin{minipage}{0.48\linewidth}
      \centering
      \includegraphics[scale=0.25]{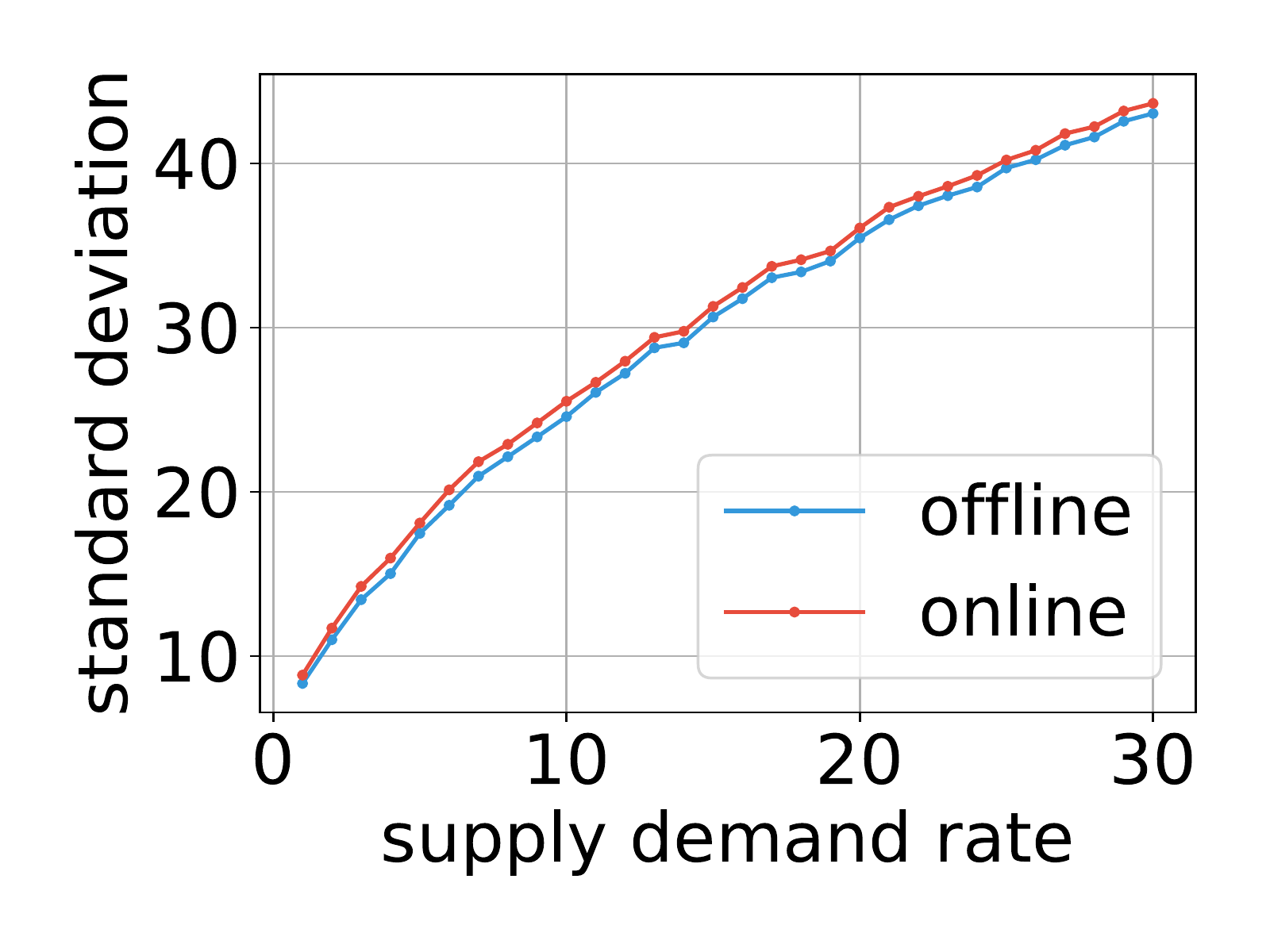}
     \end{minipage}
    \caption{\textbf{The bias and standard deviation of optimal offline experiment and online experiment related to supply-demand rate $r_1$.}}
\label{fig:online}
\end{figure}

\begin{figure}[htbp]
 \centering
  \begin{minipage}{0.5\textwidth}
   \centering
   \includegraphics[scale=0.35]{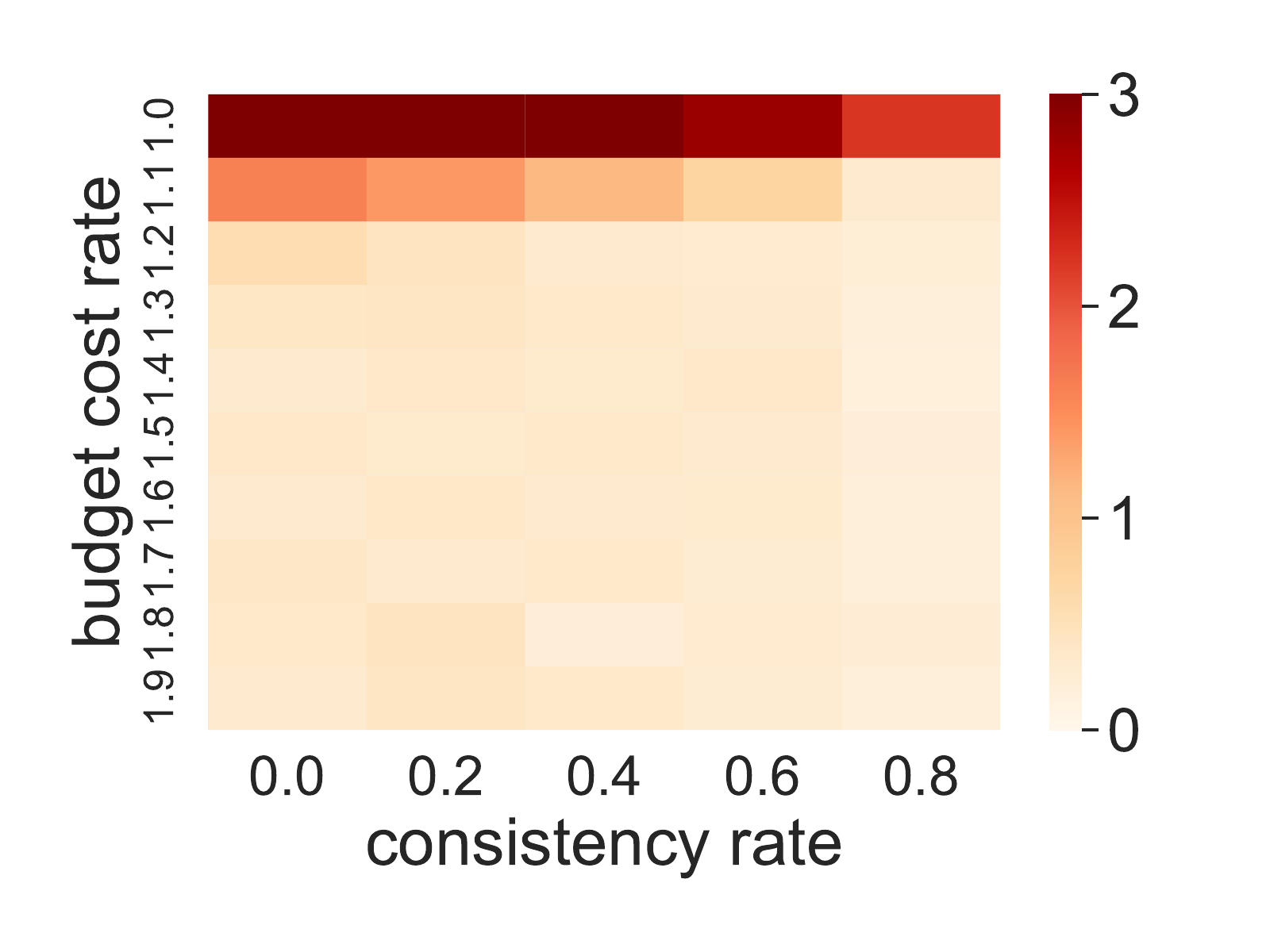}
  \end{minipage}
    \caption{\textbf{The bias of optimal experiment related to budget-cost rate $r_2$ and consistency rate $r_3=\frac{\sum_{i=1}^m \vmathbb{1}(\mathbf{w}^1_i=\mathbf{w}^0_i)}{m}$. We fix $r_1=20$.}}
\label{fig:consistent}
\end{figure}

\begin{figure}[htbp]
 \centering
  \begin{minipage}{0.48\linewidth}
   \centering
   \includegraphics[scale=0.25]{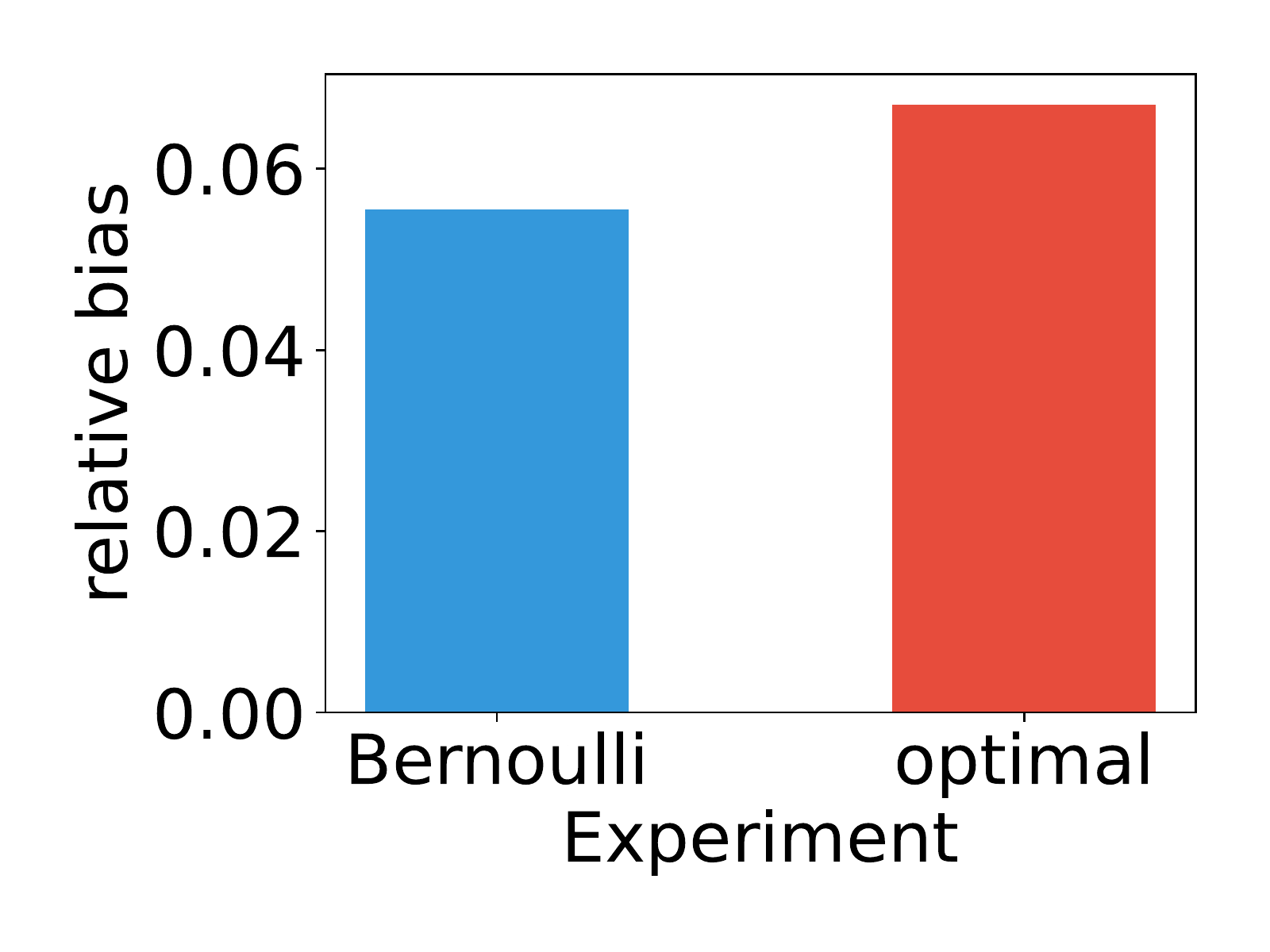}
  \end{minipage}
     \begin{minipage}{0.48\linewidth}
      \centering
      \includegraphics[scale=0.25]{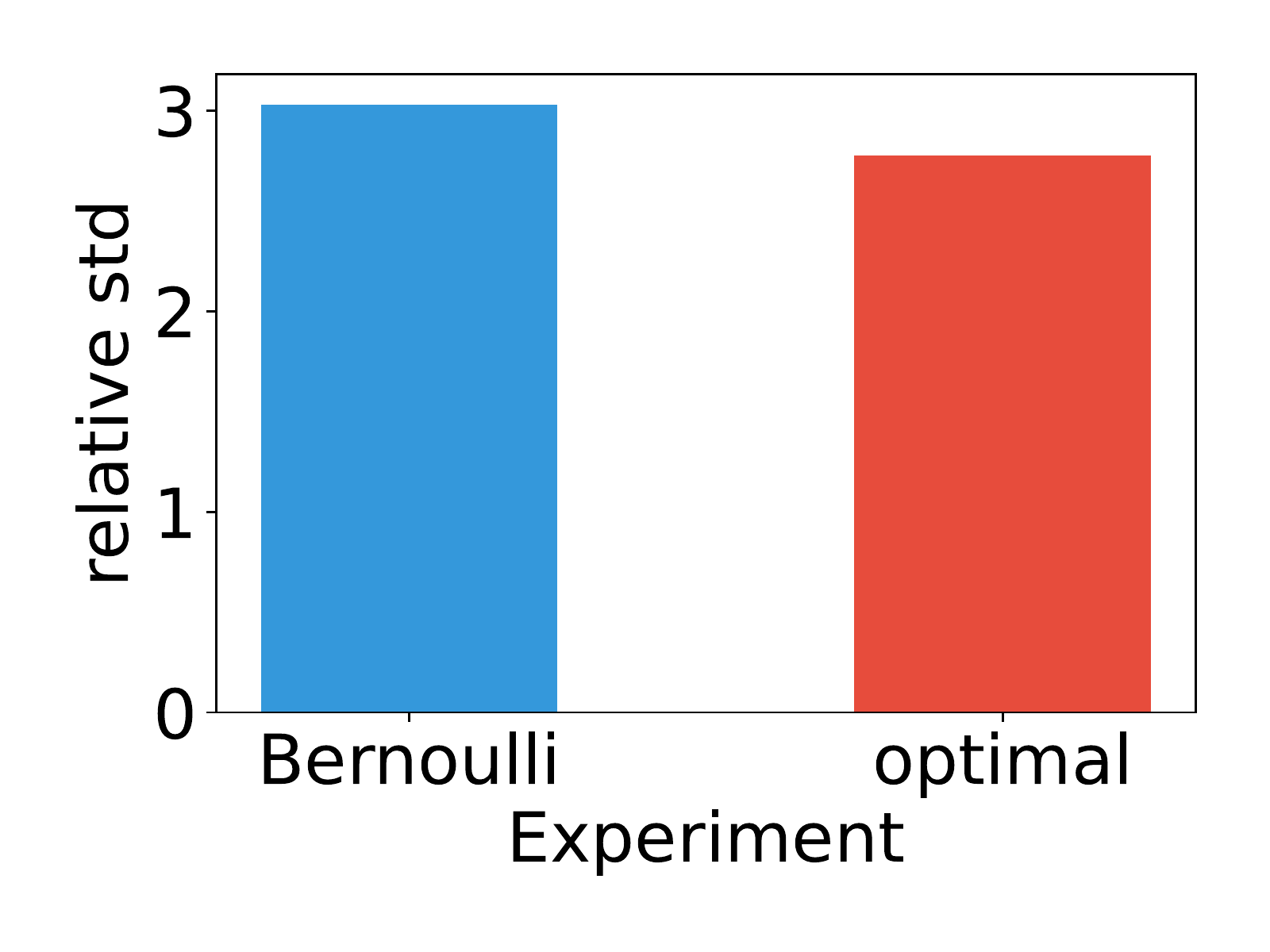}
     \end{minipage}
    \caption{\textbf{The relative bias and relative standard deviation of optimal experiment and Bernoulli randomization in real data.}}
\label{fig:real}
\end{figure}

\section{Conclusions}
In this paper, we study how to design experiments to compare two different allocations on online two-sided platforms under budget constraints. We develop a model, which includes buyers with pre-determined budgets and items with costs, with the  aim to obtain the total treatment effect (TTE), i.e., the total difference of utility under two allocation methods. Then we show how to design a feasible experiment such that the costs of each buyer do not exceed their budget. We propose estimators with small bias and nearly minimum variance. We also study the online setting and provide an online algorithm whose computation complexity is $O(m)$ times the offline setting. Finally, we test our experimental design on both synthetic data and real-world data to demonstrate the validity of our approach.

In the future, we may explore a setting where the budget can increase over time. For example, if we regard the number of rooms as the budget on Airbnb, rooms rented out will return to vacancy \cite{johari2022experimental}. This requires advanced online algorithms to handle variable budget constraints. Our current model is one-to-many, however, in many two-sided markets, both buyers and sellers have a many-to-many relationship. It may be worthwhile to impose budget constraints on both sides in these markets or conduct a two-sided experiment for improved results. Another area of research could be to examine the experiment design for strategic users who have the potential to impact the platform's costs and allocations, leading to ``interference".

\begin{acks}
Yongkang Guo, Yuqing Kong, Zhihua Zhu, and Zheng Cai were supported by the Tencent Marketing Solution RBFR202105. 

\end{acks}

\bibliographystyle{ACM-Reference-Format}
\bibliography{refs}

\clearpage
\appendix
\section{Omitted Proof}

\begin{proof}[Proof of Proposition~\ref{prop:HTestimator}]
If $p_{ij}>0$,
\begin{align*}
    &\mathbb{E}_{\mathbf{W}\sim \mathbf{X}}[\frac{o_{ij}w^1_{ij}}{p_{ij}}]\\
    =&\mathbb{E}_{\mathbf{W}\sim \mathbf{X}}[\frac{u_{ij}M(x_{ij}) w^1_{ij}}{p_{ij}}]\\
    =&\frac{u_{ij}p_{ij} w^1_{ij}}{p_{ij}}\\
    =&u_{ij}w^1_{ij}\\
\end{align*}

Similarly, $\mathbb{E}_{\mathbf{W}\sim \mathbf{X}}[\frac{o_{ij}w^0_{ij}}{p_{ij}}]=u_{ij}w^0_{ij}$.

Then we only need to show that $x_{ij}>0\Rightarrow  p_{ij}>0$. It holds because for any $M(\cdot)$ there is at least one situation for $w'_{ij}=1$, which happens when no other items are allocated to $j$.\footnote{this event happens with probability $\Pi_{i'\ne i} (1-x_{i'j})$}

Sum up over $i,j$ we will complete the proof.
\end{proof}

\begin{proof}[Proof of Theorem~\ref{thm:unbias}]

Assume variable $T_{ij}=c_{ij}$ with probability $x_{ij}$ and 0 otherwise. Since when $x_{ij}=0$, the item $i$ will have no effect on buyer $j$. Without loss of generality, we can ignore those items and assume $x_{ij}>0$ for any $i$. In the following calculation, $m$ is not the total number of items, but the number of items related to buyer $j$. In other words, $m=\sum_{i} (w_{ij}^1+w_{ij}^0)$. By Hoeffding's inequality,

$$\Pr(\sum_{i=1}^m T_{ij}-\mathbb{E}[\sum_{i=1}^m T_{ij}]\ge t)\le \exp(-\frac{2t^2}{\sum_{i=1}^m c_{ij}^2}).$$

Let $\Sigma(i)$ be the index of $i$ under random permutation and $\lceil \frac{h}{lx_0}\rceil=T$. We can get 
\begin{footnotesize}
        
\begin{align*}
   &\frac{p_{ij}}{x_{ij}}=\Pr\left(M(W)_{ij}=1|W_{ij}=1\right)\\
   =&1-\sum_{k=1}^m\frac{1}{m}\Pr\left(\sum_{\Sigma(i')< \Sigma(i)} (T_{i'j}-\mathbb{E}\left[ T_{i'j}\right])\ge B_j-\sum_{\Sigma(i')< \Sigma(i)} x_{i'j}c_{i'j}-c_{ij}\Bigg|\Sigma(i)=k\right)\\
   \ge& 1-\sum_{k=1}^m\frac{1}{m}\Pr\left(\sum_{\Sigma(i')< \Sigma(i)} (T_{i'j}-\mathbb{E}\left[T_{i'j}\right])\ge \sum_{\Sigma(i')\ge \Sigma(i)} x_{i'j}c_{i'j}-c_{ij}\Bigg|\Sigma(i)=k\right)\\
   \ge& 1-\sum_{k=1}^m\frac{1}{m}\Pr\left(\sum_{\Sigma(i')< \Sigma(i)} (T_{i'j}-\mathbb{E}\left[ T_{i'j}\right])\ge (m-k)lx_0-h\Bigg|\Sigma(i)=k\right)\\
   \ge& 1-\frac{T}{m}-\sum_{k=T+1}^m\frac{1}{m}\exp( \frac{-2(klx_0-h)^2}{\sum_{i'} c_{i'j}^2})\\
   \ge& 1-\frac{T}{m}-\sum_{k=T+1}^m\frac{1}{m}\exp( \frac{-2(klx_0-h)^2}{mh^2})\\
   \ge& 1-\frac{T}{m}-\sum_{k=T+1}^m\frac{1}{m}\exp( \frac{-2(k-T)^2l^2x_0^2}{mh^2})\\
   \ge& 1-\frac{T+m^{2/3}}{m}-\sum_{k=\lceil T+m^{2/3} \rceil}^m\frac{1}{m}\exp( \frac{-2(k-T)^2l^2x_0^2}{mh^2})\\
   \ge& 1-\frac{T+m^{2/3}}{m}-\sum_{k=\lceil T+m^{2/3} \rceil}^m\frac{1}{m}\exp( \frac{-2m^{1/3}l^2x_0^2}{h^2}).\\
   \ge& 1-\frac{T+m^{2/3}}{m}-exp( \frac{-2m^{1/3}l^2x_0^2}{h^2})
\end{align*}
\end{footnotesize}

The average bias
\begin{align*}
   &\frac{1}{m}|\mathbb{E}(\hat{\tau}(\mathbf{O}))-\tau|\\
   =&\frac{1}{m}\mathbb{E}\left|\sum_{i,j} \frac{o_{ij}(w^1_{ij}-w^0_{ij})}{x_{ij}}\left(1-\frac{p_{ij}}{x_{ij}}\right)\right|\\
   \le& \max_{i,j:x_{ij}>0} \left|\frac{u_{ij}(w^1_{ij}-w^0_{ij})}{x_{ij}}\right|\max_{i,j:x_{ij}>0}\left(1-\frac{p_{ij}}{x_{ij}}\right).\\   
\end{align*}

Thus when $m\rightarrow\infty$, $\frac{1}{m}|\mathbb{E}(\hat{\tau}(\mathbf{O}))\le O(m^{-1/3})$ and we complete the proof.

\end{proof}

\begin{proof}[Proof of Proposition~\ref{prop:varfullbudget}]
\begin{footnotesize}
    \begin{align*}
    &\mathrm{Var}(\hat{\tau}(\mathbf{O}))\\
    =&\mathrm{Var}\left(\sum_{i,j}\frac{o_{ij}(w_{ij}^1-w_{ij}^0)}{x_{ij}}\right)\\
    =&\sum_{i,j} \mathrm{Var}\left(\frac{o_{ij}(w_{ij}^1-w_{ij}^0)}{x_{ij}}\right)\\
    &+2\sum_{i,j\ne i',j'}Cov\left(\frac{o_{ij}(w_{ij}^1-w_{ij}^0)}{x_{ij}},\frac{o_{i'j'}(w_{i'j'}^1-w_{i'j'}^0)}{x_{i'j'}}\right)\\
    =&\sum_{i,j} \mathbb{E}_{u_{ij}}\left[x_{ij}\left(\frac{u_{ij}(w_{ij}^1+w_{ij}^0)}{x_{ij}}-\mu_{ij}\right)^2+(1-x_{ij})\mu_{ij}^2(w_{ij}^1+w_{ij}^0)\right]\\
    &-2\sum_{i,j,i',j'}Cov\left(\frac{o_{ij}w_{ij}^1}{x_{ij}},\frac{o_{i'j'}w_{i'j'}^0}{x_{i'j'}}\right)\\
    =&\sum_{i,j} \mathbb{E}\left[\frac{u^2_{ij}(w_{ij}^1+w_{ij}^0)}{x_{ij}}-(w_{ij}^1+w_{ij}^0)\mu_{ij}^2\right]\\
    &+2\sum_{i,j, j'}\mathbb{E}_{u_{ij},u_{ij'}}[u_{ij}u_{ij'}w_{ij}^1w_{ij'}^0]\\
    &=\sum_{i,j}\frac{(\mu^2_{ij}+\sigma_{ij}^2)(w_{ij}^1+w_{ij}^0)}{x_{ij}}-\sum_{i,j} (w_{ij}^1+w_{ij}^0)\mu_{ij}^2+2\sum_{i,j,j'}\mu_{ij}\mu_{ij'}w_{ij}^1w_{ij'}^0\\
\end{align*}
\end{footnotesize}
\end{proof}

\begin{proof}[Proof of Proposition~\ref{prop:varlimitbudget}]
Suppose $p_{ij}\in [(1-\epsilon_0)x_{ij},x_{ij}]$.
\begin{footnotesize}
\begin{align*}
    MSE(\hat{\tau}(\mathbf{O}))&=\mathbb{E}\left[\left(\sum_{i,j}\frac{o_{ij}(w_{ij}^1-w_{ij}^0)}{x_{ij}}-\mu_{ij}(w_{ij}^1-w_{ij}^0)\right)^2\right]\\
    &=\mathbb{E}\left[\left(\sum_{i,j}(\frac{o_{ij}}{x_{ij}}-\mu_{ij})w_{ij}^1-(\frac{o_{ij}}{x_{ij}}-\mu_{ij})w_{ij}^0\right)^2\right]\\
    &=\sum_{i,j} \mathbb{E}\left[\left((\frac{o_{ij}}{x_{ij}}-\mu_{ij})w_{ij}^1\right)^2\right]+\mathbb{E}\left[\left((\frac{o_{ij}}{x_{ij}}-\mu_{ij})w_{ij}^0\right)^2\right]\\ +&2\sum_{x_{ij}>0,x_{i'j'}>0,i,j\ne i',j'}\mathbb{E}\left[\left((\frac{o_{ij}}{x_{ij}}-\mu_{ij})w_{ij}^1\right)\left((\frac{o_{i'j'}}{x_{i'j'}}-\mu_{i'j'})w_{i'j'}^1\right)\right]\\
    +&2\sum_{i,j\ne i',j'}\mathbb{E}\left[\left((\frac{o_{ij}}{x_{ij}}-\mu_{ij})w_{ij}^0\right)\left(\frac{o_{i'j'}}{x_{i'j'}}-\mu_{i'j'})w_{i'j'}^0\right)\right]\\
    -&2\sum_{i,j,i',j'}\mathbb{E}\left[\left((\frac{o_{ij}}{x_{ij}}-\mu_{ij})w_{ij}^1\right)\left(\frac{o_{i'j'}}{x_{i'j'}}-\mu_{i'j'})w_{i'j'}^0\right)\right]
\end{align*}
\end{footnotesize}

Now we calculate the five terms separately. Since the first two terms are similar, we only need to calculate one of them. For the first term
\begin{align*}
    &\sum_{i,j} \mathbb{E}\left[\left((\frac{o_{ij}}{x_{ij}}-\mu_{ij})w_{ij}^1\right)^2\right]\\
    =&\sum_{i,j} w_{ij}^1\mathbb{E}\left[\left(\frac{o_{ij}}{x_{ij}}-\mu_{ij}\right)^2\right]\\
    =&\sum_{i,j} w_{ij}^1\mathbb{E}_{u_{ij}}\left[p_{ij}\frac{u^2_{ij}}{x^2_{ij}}-2\frac{u_{ij}p_{ij}}{x_{ij}}+\mu_{ij}^2\right]\\
    =&\sum_{i,j}w_{ij}^1\left(\frac{p_{ij}(\mu_{ij}^2+\sigma_{ij}^2)}{x^2_{ij}}-\frac{2\mu_{ij}p_{ij}}{x_{ij}}+\mu_{ij}^2\right)\\
    \le &\sum_{i,j}w_{ij}^1\left(\frac{\mu_{ij}^2+\sigma_{ij}^2}{x_{ij}}-2\mu_{ij}(1-\epsilon_0)+\mu_{ij}^2\right)
\end{align*}

Due to the symmetry, the second term
\begin{small}
$$\sum_{i,j} \mathbb{E}\left[\left((\frac{o_{ij}}{x_{ij}}-\mu_{ij})w_{ij}^0\right)^2\right]\le \sum_{i,j}w_{ij}^0\left(\frac{\mu_{ij}^2+\sigma_{ij}^2}{x_{ij}}-2\mu_{ij}(1-\epsilon_0)+\mu_{ij}^2\right)$$
\end{small}

For the third and fourth terms, again we only calculate the third term
\begin{align*}
    &\sum_{i,j\ne i',j'}\mathbb{E}\left[\left(\frac{o_{ij}}{x_{ij}}-\mu_{ij})w_{ij}^1\right)\left(\frac{o_{i'j'}}{x_{i'j'}}-\mu_{i'j'})w_{i'j'}^1\right)\right]\\
    =&\sum_{i,j\ne i',j'}w_{i'j'}^1w_{ij}^1\mathbb{E}\left[\left(\frac{o_{ij}}{x_{ij}}-\mu_{ij}\right)\left(\frac{o_{i'j'}}{x_{i'j'}}-\mu_{i'j'}\right)\right]\\
    =&\sum_{i,j\ne i',j'}w_{i'j'}^1w_{ij}^1\left(\mathbb{E}\left[\frac{o_{i'j'}o_{ij}}{x_{ij}x_{i'j'}}\right]-\mu_{ij}\mu_{i'j'}\left(\frac{p_{ij}}{x_{ij}}+\frac{p_{i'j'}}{x_{i'j'}}\right)\right)\\
    &+\sum_{i,j\ne i',j'}w_{i'j'}^1w_{ij}^1\mu_{ij}\mu_{i'j'}\\
    \le &\sum_{i,j\ne i',j'}w_{i'j'}^1w_{ij}^1\left(1-2(1-\epsilon_0)+1\right)\mu_{ij}\mu_{i'j'}\\
    =& \sum_{i,j\ne i',j'} 2\epsilon_0 w_{i'j'}^1w_{ij}^1\mu_{ij}\mu_{i'j'}
\end{align*}
Due to the symmetry,
\begin{footnotesize}
\begin{align*}
&\sum_{i,j\ne i',j'}\mathbb{E}[(\frac{o_{ij}}{x_{ij}}-\mu_{ij})w_{ij}^0)(\frac{o_{i'j'}}{x_{i'j'}}-\mu_{i'j'})w_{i'j'}^0)]\\
\le &\sum_{i,j\ne i',j'} 2\epsilon_0 w_{i'j'}^0w_{ij}^0\mu_{ij}\mu_{i'j'}
\end{align*}
\end{footnotesize}

For the last term
\begin{footnotesize}
\begin{align*}
-&2\sum_{i,j,i',j'}\mathbb{E}\left[\left(\frac{o_{ij}}{x_{ij}}-\mu_{ij})w_{ij}^1\right)\left(\frac{o_{i'j'}}{x_{i'j'}}-\mu_{i'j'})w_{i'j'}^0\right)\right]\\
=&-\sum_{i,j,i',j'}w_{i'j'}^0w_{ij}^1\left(\mathbb{E}\left[\frac{o_{i'j'}o_{ij}}{x_{ij}x_{i'j'}}\right]-\mu_{ij}\mu_{i'j'}\left(\frac{p_{ij}}{x_{ij}}+\frac{p_{i'j'}}{x_{i'j'}}\right)+\mu_{ij}\mu_{i'j'}\right)\\
\le &\sum_{i,j,i',j'}w_{i'j'}^0w_{ij}^1\mu_{ij}\mu_{i'j'}(1+1-1-1)\\
=0
\end{align*}
\end{footnotesize}

Since $\epsilon_0\le 1$, then add up the above inequalities we complete the proof.

\end{proof}

\end{document}